\documentclass[10pt]{article} 
\usepackage{array}

\usepackage[accepted]{rlj} 

\usepackage{amssymb}            
\usepackage{mathtools}          
\usepackage{mathrsfs}           
\usepackage{graphicx}           
\usepackage{subcaption}         
\usepackage[space]{grffile}     
\usepackage{url}                
\usepackage{lipsum}             
\usepackage{algorithm}
\usepackage{algpseudocode}
\usepackage{xcolor}
\usepackage{array}
\definecolor{FlatUIOrange}{RGB}{230,126,34}
\definecolor{FlatUiBlue}{RGB}{41,128,185}
\definecolor{FlatUIRed}{RGB}{232,65,24}
\definecolor{seabornGreen}{RGB}{85,168,104}
\newcommand{\blue}[1]{\textcolor{FlatUiBlue}{#1}}
\newcommand{\green}[1]{\textcolor{seabornGreen}{#1}}

\newcommand{\orange}[1]{\textcolor{FlatUIOrange}{#1}}
\newcommand{\oldw}{\mathbf{w}_{\text{old}}}
\newcommand{\neww}{\mathbf{w}_{\text{new}}}
\newcommand{\otheta}{\pmb{\theta}_{\text{old}}}
\newcommand{\ntheta}{\pmb{\theta}_{\text{new}}}

\title{Deep Reinforcement Learning with\\Gradient Eligibility Traces}

\setrunningtitle{Deep RL with Gradient Eligibility Traces}

\author{
Esraa Elelimy\textsuperscript{1,2,$\dagger$},
Brett Daley\textsuperscript{1,2,$\dagger$},
Andrew Patterson\textsuperscript{1,2},\\
Marlos C.\ Machado\textsuperscript{1,2,3},
Adam White\textsuperscript{1,2,3},
Martha White\textsuperscript{1,2,3}
}

\emails{\{elelimy, brett.daley, ap3, machado, amw8, whitem\}@ualberta.ca}

\affiliations{
$^{1}$\textbf{Department of Computing Science, University of Alberta, Canada}\\
$^{2}$\textbf{Alberta Machine Intelligence Institute (Amii)}\\
$^{3}$\textbf{Canada CIFAR AI Chair}
\par 
$^\dagger$ Equal contribution.
}

\contribution{
    We extend the generalized $\PBE$ to incorporate multistep credit assignment based on $\lambda$-returns, defining a new objective, the $\PBE(\lambda)$  (\Cref{sec:gpbe_lambda}).
    }
    {
    \citet{patterson2022generalized} introduced the generalized $\PBE$, which unifies and generalizes previously known objectives for value estimation. However, it was only defined for the one-step TD error.}

\contribution{
    We derive three Gradient TD algorithms that optimize our proposed objective. We derive both the forward view with the $\lambda$-return (\Cref{sec:fwd_view}) and the backward view with eligibility traces (\Cref{sec:backward_view}).
    }
    {
    Gradient TD methods were originally introduced with linear function approximation~\citep{sutton2009fast}, with a limited extension to nonlinear function approximation that required second-order information~\citep{maei2009convergent}. The recent work by~\citet{patterson2022generalized} extended these methods to non-linear function approximation without a need for second-order information. However, it was limited to the one-step TD error. 
    }

\contribution{
    We introduce Gradient PPO, a policy gradient algorithm that uses our sound forward-view value estimation algorithms (\Cref{sec:gppo}).
    }
    {
    PPO \citep {schulman2017proximal} is a widely-used policy gradient method that relies on semi-gradient TD updates for value estimation. We build on PPO by replacing the value estimation component with a new one that uses Gradient TD methods. This change required non-trivial modification to PPO, resulting in our new algorithm, Gradient PPO. Gradient PPO is the first policy gradient method that uses Gradient TD algorithms in a deep RL setting with a replay buffer.
    }
\contribution{We introduce QRC($\lambda$), which uses our backward-view eligibility traces and is suitable for streaming settings. (\Cref{sec:backward_view}).}
{Backward-view algorithms can make updates on each time step without delay, making them efficient in streaming settings~\citep{elsayed2024streaming}. QRC($\lambda$) is the first backward-view algorithm that uses Gradient TD methods in the streaming deep RL. }

\summary{
Achieving fast and stable off-policy learning in deep reinforcement learning (RL) is challenging.
Most existing methods rely on semi-gradient temporal-difference (TD) methods for their simplicity and efficiency, but are consequently susceptible to divergence.
While more principled approaches like Gradient TD (GTD) methods have strong convergence guarantees, they have rarely been used in deep RL.
Recent work introduced the generalized Projected Bellman Error ($\PBE$), enabling GTD methods to work efficiently with nonlinear function approximation.
However, this work is limited to one-step methods, which are slow at credit assignment and require a large number of samples.
In this paper, we extend the generalized $\PBE$ objective to support multistep credit assignment based on the $\lambda$-return and derive three gradient-based methods that optimize this new objective.
We provide both a forward-view formulation compatible with experience replay and a backward-view formulation compatible with streaming algorithms.
Finally, we evaluate the proposed algorithms and show that they outperform both PPO and StreamQ in MuJoCo and MinAtar environments, respectively. \footnote{Code available at \url{https://github.com/esraaelelimy/gtd\_algos}}

}

\usepackage{paper_style}  

\begin{document}

\makeCover  
\maketitle  

\begin{abstract}
Achieving fast and stable off-policy learning in deep reinforcement learning (RL) is challenging.
Most existing methods rely on semi-gradient temporal-difference (TD) methods for their simplicity and efficiency, but are consequently susceptible to divergence.
While more principled approaches like Gradient TD (GTD) methods have strong convergence guarantees, they have rarely been used in deep RL.
Recent work introduced the generalized Projected Bellman Error ($\PBE$), enabling GTD methods to work efficiently with nonlinear function approximation.
However, this work is limited to one-step methods, which are slow at credit assignment and require a large number of samples.
In this paper, we extend the generalized $\PBE$ objective to support multistep credit assignment based on the $\lambda$-return and derive three gradient-based methods that optimize this new objective.
We provide both a forward-view formulation compatible with experience replay and a backward-view formulation compatible with streaming algorithms.
Finally, we evaluate the proposed algorithms and show that they outperform both PPO and StreamQ in MuJoCo and MinAtar environments, respectively. \footnote{Code available at \url{https://github.com/esraaelelimy/gtd\_algos}}

\end{abstract}

\section{Introduction}

Estimating the value function is a fundamental component of most RL algorithms.
All value-based methods depend on estimating the action-value function for some target policy and then acting greedily with respect to those estimated values.
Even in policy gradient methods, where a parameterized policy is learned, most algorithms learn a value function along with the policy.
Many RL algorithms use semi-gradient temporal-difference (TD) learning algorithms for value estimation, despite known divergence issues under nonlinear function approximation \citep{tsitsiklis1997analysis} and under off-policy sampling \citep{baird1995residual}, both of which frequently arise in modern deep RL settings.
\looseness=-1

There have been significant advances towards deriving sound off-policy TD algorithms.
For a brief history, the mean squared Bellman error ($\BE$) was an early objective, which produces a different solution from the TD fixed point but similarly aims to satisfy the Bellman equation.
However, the $\BE$ was not widely used because it is difficult to optimize without a simulator due to the double-sampling problem \citep{baird1995residual}.
The mean squared \emph{projected} Bellman error for linear function approximation---which we call the linear $\PBE$---was introduced later, and a class of Gradient TD methods was derived to optimize this objective \citep{sutton2009fast}. 
An early attempt to extend Gradient TD methods to nonlinear function approximation made the assumption that the updates to the parameters of the value function in each step are small, and thus the nonlinear manifold of the value functions can be treated as locally linear \citep{maei2009convergent}. 
Recently, \cite{patterson2022generalized} introduced a generalization of the $\PBE$ objective that is based on the conjugate form of the $\BE$ \citep{dai2017learning}. This new objective made it possible to derive Gradient TD methods for the nonlinear setting, while still having an equivalence to the previous Gradient TD methods in the case of linear function approximation.
This generalized objective was further extended to allow for robust losses in the Bellman error \citep{patterson2023robust} and is a promising avenue for the development of sound value-estimation algorithms. However, it has not been extended to include multi-step updates. In the rest of the paper, we refer to this generalized $\PBE$ objective as simply the $\PBE$ objective and use the term linear $\PBE$ when referring to the previous linear-only objective.

\begin{table}[t]
    \caption{
        Related Gradient TD literature. Our paper is the first to define and optimize the generalized PBE($\lambda$) objective for nonlinear function approximation (see \Cref{sec:fwd_view}).
    }
    \begin{center}
    \resizebox{\textwidth}{!}{
        \begin{tabular}{>{\centering\arraybackslash}m{0.22\textwidth} 
                        >{\centering\arraybackslash}m{0.28\textwidth} 
                        >{\centering\arraybackslash}m{0.22\textwidth} 
                        >{\centering\arraybackslash}m{0.28\textwidth}}
            \toprule
            \multicolumn{2}{c}{\bf Linear Function Approximation} & \multicolumn{2}{c}{\bf Nonlinear Function Approximation} \\\addlinespace
            {\bf One-step} & {\bf $\lambda$-return} & {\bf One-step} & {\bf $\lambda$-return} \\
            \midrule
            \citep{sutton2009fast} & \citep{maei2010gq} & \citep{maei2009convergent,patterson2022generalized} & Our paper \\
            \bottomrule
        \end{tabular}
    }
    \end{center}
    \label{tab:related_lit}
\end{table}

In this paper, we extend the $\PBE$ to incorporate multistep credit assignment using $\lambda$-returns.
\Cref{tab:related_lit} summarizes the algorithmic gaps that we fill.
We derive similar gradient variants as were derived for the one-step $\PBE$ \citep{patterson2022generalized}, but now also need to consider forward-view and backward-view updates for our proposed objective, $\PBE$($\lambda$). 
We introduce Gradient PPO, a policy gradient algorithm that modifies PPO to use our sound forward-view value estimation algorithms. We show that Gradient PPO significantly outperforms PPO in two MuJoCo environments and is comparable in two others.
We also introduce QRC($\lambda$), which uses backward-view (i.e., eligibility trace updates) and is suitable for online streaming settings.\footnote{A setting motivated by hardware limitations where we replay buffers are not used and updates are made one sample at a time. See \citet{elsayed2024streaming}.}
We show that QRC($\lambda$) is significantly better in all MinAtar environments than StreamQ \citep{elsayed2024streaming}, a recent algorithm combining Q($\lambda$) with a new optimizer and an initialization scheme for better performance in streaming settings.
We investigate multiple variants of our forward-view and backward-view algorithms; as was concluded for $\PBE$(0) \citep{ghiassian2020gradient,patterson2022generalized}, we find that a variant based on regularized corrections called TDRC consistently outperforms the other variants.
This work provides a clear way to incorporate gradient TD methods with eligibility traces into deep RL methods and offers two new promising algorithms that perform well in practice. 

\section{Background}
\label{sec:background}

We consider the Markov Decision Process (MDP) formalism where the agent-environment interactions are described by the tuple $(\S, \A, p, \mathcal{R})$.
At each time step, $t= 1,2,3,\ldots$, the agent observes a state, $S_t \in \S$, and takes an action, $A_t \in \A$, according to a policy $\pi: \S \times \A \rightarrow {[0,1]}$, where $\S$ and $\A$ are finite sets of states and actions, respectively.
Based on $S_t$ and $A_t$, the environment transitions to a new state, $S_{t+1} \in \S$, and yields a reward, $R_{t+1} \in \mathcal{R}$, with probability $p(S_{t+1},R_{t+1} \mid S_t,A_t)$. The value of a policy is defined as 
$\smash{v_\pi(s) \defeq \E_\pi [G_t \mid S_t = s]}, \forall s \in \S$, 
where the return, $\smash{G_t \defeq \sum_{i=0}^\infty \gamma^i R_{t+1+i}}$, is the discounted sum of future rewards from time $t$ with discount factor $\gamma \in [0,1]$. 

The agent typically estimates the value function using 
a differentiable parameterized function, such as a neural network.
We define the parameterized value function as $\hat{v}(s,\vw) \approx v_\pi(s)$, where $\vw \in \R^{d_\vw}$ is a weight vector and $d_\vw < |\S|$.
One objective that can be used to learn this value function is the mean squared Bellman error ($\BE$):
\begin{equation}~\label{eq:be1}
    \BE(\vw) \defeq \sum_{s \in \S} d(s) \, \E_\pi [\delta (s) \mid S=s]^2
    \,,
\end{equation}
where $d$ is the state distribution\footnote{Note that we write the expectation with a sum to make the notation more accessible, but this can be generalized to continuous state spaces using integrals.} and $\delta$ is the TD error for a transition $(S,A,S',R)$.
The $\delta$ can be different depending on the algorithm.
For state-value prediction, we use $\delta \defeq R + \gamma \hat{v}(S',\vw) - \hat{v}(S,\vw)$.
For control, to learn optimal action-values $q_*(s,a)$, we use $\delta \defeq R + \gamma \max_{a' \in \A} \hat{q}(S',a',\vw) - \hat{q}(S,A,\vw)$.
For control, we would additionally condition on $A = a$ and sum over $(s,a)$ instead of $s$ in \Cref{eq:be1}, but for simplicity of exposition, we only show the objectives for $\hat{v}$. 
We cannot generally reach zero $\BE$, unless the true values are representable by our parameterized function class for all states with nonzero weight. Additionally, the $\BE$ objective is difficult to optimize, due to the double sampling and identifiability issues \citep{sutton2018reinforcement}, and we instead consider a more practical objective called the $\PBE$.

The $\PBE$ objective introduced by \cite{patterson2022generalized} generalizes and unifies several objectives and extends Gradient TD methods to nonlinear function approximation. 
The $\PBE$ objective builds on prior work \citep{dai2017learning} that avoids the double sampling by reformulating the $\BE$ using its conjugate form with an auxiliary variable $h$. Using the fact that the biconjugate of a quadratic function is $x^2 = \max_{h \in \R} 2xh - h^2$, we can re-express the $\BE$ as
\begin{equation}
    \label{eq:be}
    \BE(\vw) \defeq \max_{h \in \mathcal{F}_{\text{all}}} \sum_{s \in S} d(s) \Big( 2 \, \delta_\pi(s) \, h(s) - h(s)^2 \Big)
    \,,
\end{equation}
where $\mathcal{F}_{\text{all}}$ is the space of all functions and $\delta_\pi(s) \defeq \E_\pi[\delta_t \mid S_t=s]$.
For a state $s$, the optimal $h^*(s) = \delta_\pi(s)$, and we recover the $\BE$.
More generally, we can learn a parameterized function that approximates this auxiliary variable, $h$.
Letting $\H$ be the space of the parameterized functions for $h$, the $\PBE$ then projects $\BE$ into $\H$, and is defined as:
\begin{equation}
    \label{eq:gpbe}
    \PBE(\vw) \defeq \max_{h \in \H} \sum_{s \in S} d(s) \left( 2 \, \delta_\pi(s) \, h(s) - h(s)^2 \right)
    \,.
\end{equation}
Depending on the choice of $\H$, the $\PBE$ can express a variety of objectives.
For a linear function class, we recover the linear $\PBE$, and for a highly expressive function class, we recover the (identifiable) $\BE$ \citep{patterson2022generalized}.

The $\PBE$ can be optimized by taking the gradient of \Cref{eq:gpbe}, which results in a saddle-point update called \emph{GTD2}.
Alternatively, we can do a gradient correction update, which results in the empirically preferable algorithm called \emph{TDC}.
Note that GTD2 and TDC were introduced for the linear setting \citep{sutton2009fast}, but the same names are used when generalized to the nonlinear setting \citep{patterson2022generalized}, so we follow that convention. 
TDC has been shown to outperform GTD2 \citep{ghiassian2020gradient,white2016investigating,patterson2022generalized} and has been further extended to include a regularization term, resulting in a better update called TDRC \citep{ghiassian2020gradient,patterson2022generalized}.
\looseness=-1

We briefly include the update rule for these three Gradient TD methods, as we will extend them in the following sections.
For $\hat{v}$ parameterized by $\vw$ and $\hat{h}$ parameterized by $\vtheta$, 
all methods can be written as jointly updating
\begin{align}
\begin{split}
    \vw_{t+1} &\gets \vw_t + \alpha \, \Delta \vw_t
    \,,\\
    \vtheta_{t+1} &\gets \vtheta_t + \alpha \, \Delta \vtheta_t
    \,,
\end{split}
\end{align}
where $\alpha \in (0,1]$ is a step-size hyperparameter---or, more generally, an optimizer like Adam \citep{kingma2014adam} can be used. 
For GTD2, $\Delta \vw_t$ is
\begin{equation*}
    \Delta \vw_t = -\hat{h}(S_t, \vtheta_t)\nabla_{\vw} \delta_t  =  \hat{h}(S_t, \vtheta_t) \big( \nabla_{\vw}\hat{v}(S_t,\vw) - \gamma \nabla_{\vw} \hat{v}(S_{t+1},\vw) \big) 
    \,.
\end{equation*}
The TDC update replaces the term $\hat{h}(S_t, \vtheta_t)\nabla_{\vw}\hat{v}(S_t,\vw)$ with $\delta_t\nabla_{\vw}\hat{v}(S_t,\vw)$, to get the update
\begin{equation*}
    \Delta \vw_t = \delta_t\nabla_{\vw}\hat{v}(S_t,\vw) - \hat{h}(S_t, \vtheta_t) \nabla_{\vw} \gamma \hat{v}(S_{t+1},\vw) 
    \,.
\end{equation*}
This update is called TD with corrections, because the first term is exactly the TD update and the second term acts like a correction to the semi-gradient TD update. This modified update is motivated by noting that $h^*(s) = \delta_\pi(s)$, and so replacing the approximation $\hat{h}(S_t, \vtheta_t)$ with an unbiased sample $\delta_t$ instead is sensible. TDC has been shown to converge to the same fixed point as TD and GTD2 in the linear setting \citep{maei2011gradient}. 

Both GTD2 and TDC have the same $\Delta \vtheta_t$ which can be written as $\Delta \vtheta_t = \left(\delta_t - \hat{h}(S_t, \vtheta_t)\right) \nabla_{\vtheta}\hat{h}(S_t, \vtheta_t).$
TDRC uses the same $\Delta \vw_t$ as TDC, but regularizes the auxiliary variable:
\begin{equation*}
    \Delta \vtheta_t = \left(\delta_t - \hat{h}(S_t, \vtheta_t)\right) \nabla_{\vtheta}\hat{h}(S_t, \vtheta_t) - \beta \vtheta_t
    \,,
\end{equation*}
where $\beta \in [0, \infty)$.
For $\beta = 0$, TDRC is the same as TDC.
As $\beta$ is increased, $h$ gets pushed closer to zero and TDRC becomes closer to TD. TDRC was found to be strictly better than TDC, even with a fixed $\beta = 1$ across several problems~\citep{ghiassian2020gradient,patterson2022generalized}. This improvement was further justified theoretically with a connection to robust Bellman losses \citep{patterson2023robust}, motivating regularization on $h$.

\section{The Generalized PBE($\lambda$) Objective}~\label{sec:gpbe_lambda}

The basis of the $\PBE$ is the $1$-step TD error, which means that credit assignment can be slow. Reward information must propagate backward one step at a time through the value function, via bootstrapping.
In this section, we extend the $\PBE$ to incorporate multistep credit assignment using the $\lambda$-return.

First, let us define our multistep target. The simplest multistep return estimator is the $n$-step return, defined as
\begin{equation*}
    \nstep{n}_t \defeq \Big( \sum_{i=0}^{n-1} \gamma^i R_{t+1+i} \Big) + \gamma^n \hat{v}(S_{t+n},\vw_t)
    \,.
\end{equation*}

The $\lambda$-return is the exponentially weighted average of all possible $n$-step returns:
\begin{equation}
    \label{eq:lambda-return}
    G^{\lambda}_t \defeq (1-\lambda) \sum_{n=1}^{\infty} \lambda^{n-1} \nstep{n}_t
    \,,
\end{equation}
where $\lambda \in [0,1]$.
The $\lambda$-return is the return target for TD($\lambda$) \citep{sutton1988learning} and comes with a number of desirable properties:
it smoothly interpolates between TD and Monte Carlo methods \citep[a bias-variance trade-off;][]{kearns2000bias}, reduces variance compared to a single $n$-step return \citep{daley2024averaging}, and imposes a recency heuristic by assigning less weight to temporally distant experiences \citep{daley2024demystifying}.
We denote the error between the $\lambda$-return target and the current value estimate by
\begin{equation}
    \label{eq:td_lambda_error}
    \delta^\lambda_t
    \defeq G^{\lambda}_{t} - \hat{v}(S_t,\vw_t)
    = \sum_{i=0}^\infty (\gamma \lambda)^i \delta_{t+i}
    \,,
\end{equation}
and refer to this quantity as the TD($\lambda$) error.
We note that in the context of recent works, the TD($\lambda$) error is often referred to as the generalized advantage estimate \citep[GAE;][]{schulman2015high}.

We start by defining $\BE$($\lambda$) using the TD($\lambda$) error.
For $\delta^\lambda_\pi(s) \defeq \E_\pi[\delta^\lambda_t \mid S_t=s]$, we define the $\BE$($\lambda$) analogously to \Cref{eq:be1} as
\begin{equation*}
    \BE(\vw,\lambda) \defeq \sum_{s \in \S} d(s) \, \delta^\lambda_\pi(s)^2
    \,.
\end{equation*}
Following the derivation of the $\PBE$ in \Cref{eq:gpbe}, with the definitions of $h$ and the new $\delta^\lambda_\pi(s)$, we can write the $\PBE$($\lambda$) objective as
\begin{equation}
    \label{eq:gpbe_lambda}
    \PBE(\vw,\lambda) \defeq \max_{h \in \H} \sum_{s \in S} d(s) \Big( 2 \delta^\lambda_\pi(s) \, h(s) - h(s)^2 \Big)
    \,.
\end{equation}
When $\lambda=0$, we recover the original one-step $\PBE$ objective \citep{patterson2022generalized}.
In the absence of function approximation, the $\PBE$ and the $\PBE$($\lambda$) objectives lead to the same solution, $v_\pi$, because their fixed points are both $v_\pi$. However, under function approximation when we cannot perfectly represent $v_\pi$, the choice of $\lambda$ impacts the minimum-error solution.
In practice, intermediate $\lambda$-values on the interval $(0,1)$ will balance between solution quality, learning speed, and variance.

\section{The Forward-View for Gradient TD($\lambda$) Methods}
\label{sec:fwd_view}
In this section, we develop several forward-view methods for optimizing the $\PBE$($\lambda$) under nonlinear function approximation.
Following the previous convention, we will overload the names GTD2($\lambda$) and TDC($\lambda$) introduced for the linear setting because we are strictly generalizing them to a broader function class. 

\textbf{GTD2(\texorpdfstring{$\lambda$}{lambda}):}
We derive this algorithm by taking the gradient of \Cref{eq:gpbe_lambda}
w.r.t.\ to both $\vw$ and $\vtheta$.
\begin{align*}
    \frac{1}{2} \grad{\vw} \sum_{s \in \S} d(s) \Big(2 \, \delta^\lambda_\pi(s) \, h(s) - h(s)^2 \Big)
    &= \sum_{s \in \S} d(s) h(s) \grad{\vw} \delta^\lambda_\pi(s)
    \,,\\
    \frac{1}{2} \grad{\vtheta} \sum_{s \in \S} d(s) \Big(2 \, \delta^\lambda_\pi(s) \, h(s) - h(s)^2 \Big)
    &= \sum_{s \in \S} d(s) \big( \delta^\lambda_\pi(s) - h(s) \big) \grad{\vtheta} h(s)
    \,.
\end{align*}
We get a stochastic gradient descent update by sampling these expressions. For brevity throughout, let $V_t \defeq \hat{v}(S_t,\vw_t)$ and $H_t \defeq \hat{h}(S_t,\vtheta_t)$. The resulting update is then
\begin{align}
    \label{eq:b-gtd2_forward_main}
    \Delta \vw_t &= - H_t \grad{\vw} \delta^\lambda_t
    \,,\\
    \label{eq:b-gtd2_forward_aux}
    \Delta \vtheta_t &= (\delta^\lambda_t - H_t) \grad{\theta} H_t
    \,.
\end{align}
GTD2($\lambda$) is a standard saddle-point update and should converge to a local optimum of the $\PBE$($\lambda$).

\textbf{TDC($\lambda$):}
For TDC(0), we obtained an alternative gradient correction by adding the term $(\delta_t - h(S_t))\grad{\vw} \hat{v}(S_t, \vw)$ to the GTD2(0) update.
This was motivated by the fact that $h(S_t)$ approximates $\delta_t$.
We take a similar approach here, adding $(\delta^\lambda_t - H_t) \grad{\vw} \hat{v}(S_t,\vw_t)$ to the GTD2($\lambda$) update for $\vw$:
\looseness=-1
\begin{align}
    \nonumber
    \Delta \vw_t &= (\delta^\lambda_t - H_t) \grad{\vw} V_t - H_t \grad{\vw} \delta^\lambda_t\\
    \label{eq:tdc_main}
    &= \delta^\lambda_t \grad{\vw} V_t - H_t \grad{\vw} (V_t + \delta^\lambda_t)
    \,.
\end{align}
The $\vtheta$-update remains the same as \Cref{eq:b-gtd2_forward_aux}.
The result is the sum of a semi-gradient TD($\lambda$) update and a gradient correction.
However, the method is biased, as it assumes that $H_t$ has converged exactly to $\delta^\lambda_\pi(S_t)$.
This bias did not impact convergence of TDC in the linear setting, but as yet there is no proof of convergence of TDC in the nonlinear setting.
Similarly, it is not yet clear what the ramifications are of using TDC($\lambda$) rather than GTD2($\lambda$), although, in our experiments, we find it is better empirically.

\textbf{TDRC($\lambda$):} Finally, we extend the TDRC algorithm, and the extension simply involves adding a regularization penalty with coefficient $\beta \geq 0$ to the update for $h$:
\begin{equation}
    \Delta \vtheta_t = (\delta^\lambda_t - H_t) \grad{\vtheta} H_t - \beta \vtheta_t
    \,.
    \label{eq:tdrc_theta}
\end{equation}
All the methods we derived in this section depend on the forward-view of the $\lambda$-return from \Cref{eq:lambda-return}, which means they need a trajectory of transitions to make an update.
This makes these methods appealing when there is a replay buffer to store and sample these trajectories.
Further, the trajectories should be on-policy to avoid the need to incorporate importance sampling ratios.
It is not difficult to incorporate importance sampling (we include these extensions in \Cref{app:is_correc}), but significant variance may arise when using importance sampling, which degrades performance in practice.
In the next section, we incorporate these forward-view updates into PPO, an algorithm that makes use of both replay and on-policy trajectories.

\section{Gradient PPO: Using the Forward-View in Deep RL}
\label{sec:gppo}
In this section, we introduce a new algorithm, called Gradient PPO, that modifies the PPO  algorithm \citep{{schulman2017proximal}} to incorporate the forward-view gradient updates derived in the last section.

\subsection{Gradient PPO}
Proximal Policy Optimization \citep[PPO;][]{schulman2017proximal} is a widely used policy-gradient method that learns both a parameterized policy, the actor, and an estimate for the state-value function, the critic.
In PPO, the agent alternates between collecting a fixed-length trajectory of interactions and performing batch updates using that trajectory to learn both the policy and the state-value function.
We will focus on the critic component of PPO, as that is the part learning the value function, and we will modify it to use the gradient-based methods introduced in \Cref{sec:fwd_view}.

PPO updates depend on the Generalized Advantage Estimate \citep[GAE;][]{schulman2015high}, which is identical to the TD($\lambda$) error in \Cref{eq:td_lambda_error}.
In practice, however, PPO updates must truncate the GAE due to the finite length of the collected experience trajectory.
Given a trajectory of length $T$, the truncated GAE can be written as $\delta^\lambda_{t:T} = \sum_{i=0}^{T-t-1} (\gamma \lambda)^i \delta_{t+i}$, and we can form an estimate for the $\lambda$-return using that truncated GAE as:
\begin{equation}
    G^{\lambda}_{t:T} \defeq \hat{v}(S_t, \vw) + \sum_{k=0}^{T-t-1} (\gamma \lambda)^k \delta_{t+k}
    \,.
\end{equation}
The value-function objective for PPO can then be written as follows:
\begin{equation}
~\label{eq:ppo_obj}
    L^{\text{\tiny{PPO}}}_t(\vw_t) = \frac{1}{2} \left(\hat{v}(S_t, \vw_t) - G^{\lambda}_{t:T}\right)^2
    \,,
\end{equation}
PPO typically uses a stale target for $G^{\lambda}_{t:T}$. i.e., the $\lambda$-return target is computed once from the collected trajectory and is kept fixed for all the training epochs on that trajectory.
Although many PPO implementations heuristically clip this loss, we remove this component from our algorithm for simplicity.

We now introduce \emph{Gradient PPO}, which changes the critic update for PPO to allow for Gradient TD($\lambda$) updates.
Gradient PPO introduces the following three changes.

\textbf{Modification 1:}
We change the updates to the critic parameters, $\vw$, to use the updates in \Cref{eq:tdc_main}.

\textbf{Modification 2:} We introduce an auxiliary network with parameters $\vtheta$ to learn the auxiliary variable $\hat{h}$, and update the parameters using \Cref{eq:tdrc_theta}.\footnote{An alternative to implementing the direct parameter updates is to write corresponding loss functions and use automatic differentiation libraries to compute the updates. We provide the corresponding loss functions to our updates in \Cref{app:algos}.}
    
\textbf{Modification 3:}
We need to compute the gradient for $\delta^{\lambda}_{t}$.
As a result, we cannot use a stale target as in \Cref{eq:ppo_obj}.
Instead, we need to recompute $\delta^{\lambda}_t$ and its gradient after each update. We do this by sampling sequences from the minibatch instead of sampling independent samples.
We then compute a truncated $\delta^{\lambda}_{t:\tau}$ based on the sampled sequences.
In this case, the effective truncation for the $\lambda$-return is the length of the sequence sampled from a minibatch, $\tau$, rather than the full trajectory length $T$.
A similar approach to incorporate the $\lambda$-return with replay buffers was previously introduced \citep{daley2019reconciling}.
This approach might seem computationally expensive at first since $\vw$ is used to compute all the values included in $\hat{\delta}^\lambda_{t:\tau}$ estimation.
However, a nice property of the gradient $\nabla_{\vw}\hat{\delta}^\lambda_{t:\tau}$ is that it can be easily computed recursively as follows:
\begin{equation*}
    \nabla_{\vw}\delta_t^\lambda = \gl \nabla_{\vw} \delta^\lambda_{t+1} + \nabla_{\vw} \delta_t.
\end{equation*}
Then, given a sequence of length $\tau$, $\nabla_{\vw}\delta_t^\lambda$ and $\delta_t^\lambda$ can be estimated using \Cref{alg:tdrc_estimation}, where lines in \green{green} highlight the additional computations required for Gradient PPO per a minibatch update. 

Implementations for Gradient PPO can simply pass the newly defined loss functions, \Cref{eq:gppo_1} and \Cref{eq:gppo_2}, directly to an automatic differentiation. But implementations based on \Cref{alg:tdrc_estimation} might be more efficient as it allows for parallel computations of the values for all states. We show in \Cref{app:gppo} that using this parallel approach for computation results in Gradient PPO having the same Steps Per Second (SPS) cost as PPO; there is no drop in runtime from these additional calculations.
We also provide a full algorithm for PPO and Gradient PPO in \Cref{app:algos}. 

\begin{algorithm}[!hbt]
    \caption{Estimating TDRC($\lambda$) Updates for Gradient PPO}~\label{alg:tdrc_estimation}
    \begin{algorithmic}
    \State{Input: A sequence of states, $s_t,\ldots s_{t+\tau}$.}
    \State{Input: The current weight parameters of the value function, $\vw$. }
    \State{For all samples in the sequence, compute $\hat{v}(s_t, \vw)$ and $\nabla_{\vw}\hat{v}(s_t, \vw)$.}\\ \Comment{This step is done in parallel by creating a batch of all observations.}
    \For{$j=t+\tau -1$, \ldots, t} 
        \State{$\delta_j = R_{j+1} + \gamma \hat{v}(s_{j+1},\vw) - \hat{v}(s_{j},\vw)$}
        \State{\green{$\nabla{\delta}_{j} = R_{j+1} + \gamma \nabla{\hat{v}}(s_{j+1},\vw) - \nabla{\hat{v}}(s_{j},\vw)$}}
        \State{$\delta_j^\lambda = \delta_j + \gamma \lambda \delta_{j+1}^\lambda$}
        \State{\green{$\nabla{\delta}_j^{\lambda} = \nabla{\delta}_j + \gamma \lambda \nabla{\delta}_{j+1}^{\lambda}$}}
    \EndFor{}
    \end{algorithmic}
\end{algorithm}

\subsection{Empirical Investigation of Gradient PPO}
We evaluate the performance of Gradient PPO across several environments from the MuJoCo Benchmark~\citep{todorov2012mujoco}. For Gradient PPO, we performed a hyperparameter sweep for the actor learning rate, the critic learning rate, and $\lambda$. For the auxiliary variable $h$, we used the same learning rate as the critic. We tested each hyperparameter configuration on all environments and repeated the experiments across $5$ seeds.
Finally, based on the sweep results, we selected a single hyperparameter configuration that worked reasonably well across all environments and evaluated it for $30$ more seeds---the {\em two-stage approach} \citep{patterson2024empirical}. We provide the ranges of values we swept over in \Cref{app:gppo} and the hyperparameters configuration that we will use in all Gradient PPO experiments in \Cref{tab:gppo_hypers}. For PPO, we used the default hyperparameters commonly used for PPO with MuJoCo environments~\citep{shengyi2022the37implementation}, listed in Table~\ref{tab:ppo_hypers}. 

Figure~\ref{fig:ppo_mujoco_main_res} shows the Gradient PPO and Default PPO results across four MuJoCo environments. In Ant and HalfCheetah, Gradient PPO clearly outperforms PPO. Both algorithms perform similarly in Walker and Hopper.

\begin{figure}[htb]
    \centering
    \includegraphics[width=\linewidth]{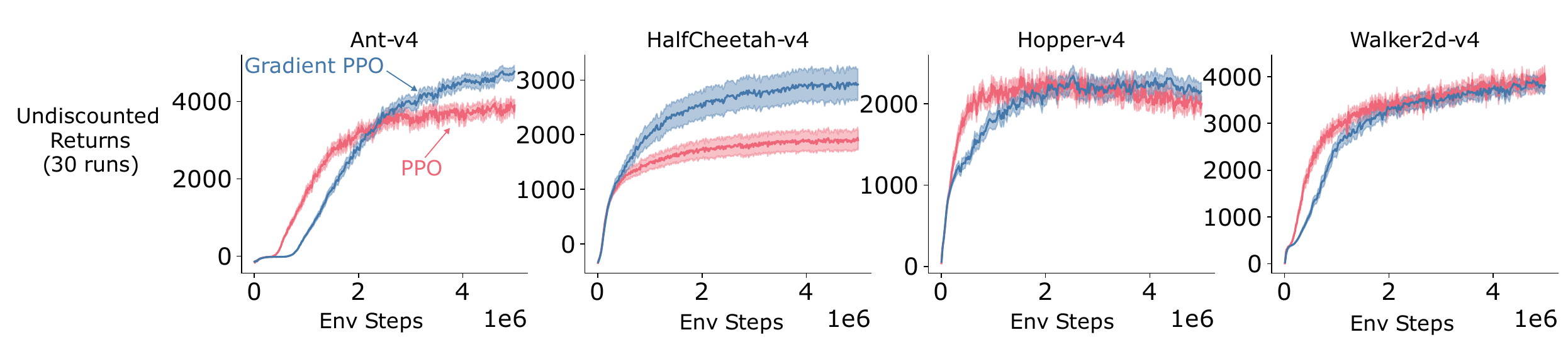}
    \caption{Gradient PPO and PPO evaluated on four MuJoCo environments. The solid lines are the mean performance averaged over $30$ seeds, and the shaded area is the standard error.}
    \label{fig:ppo_mujoco_main_res}
\end{figure}
We also investigated the utility of using TDRC($\lambda$) instead of TDC($\lambda$) and GTD2($\lambda$) to estimate the critic. 
\Cref{fig:ppo_mujoco_var} shows the results with these variations. There is a marked difference in performance, suggesting that both gradient corrections and regularization are needed for improved performance when using gradient-based methods. This outcome aligns with our discussion in \Cref{sec:background} and \Cref{sec:fwd_view} about how TDRC has been shown to outperform TDC, which in turn outperforms GTD2.

\begin{figure}[htb]
    \centering
    \includegraphics[width=\linewidth]{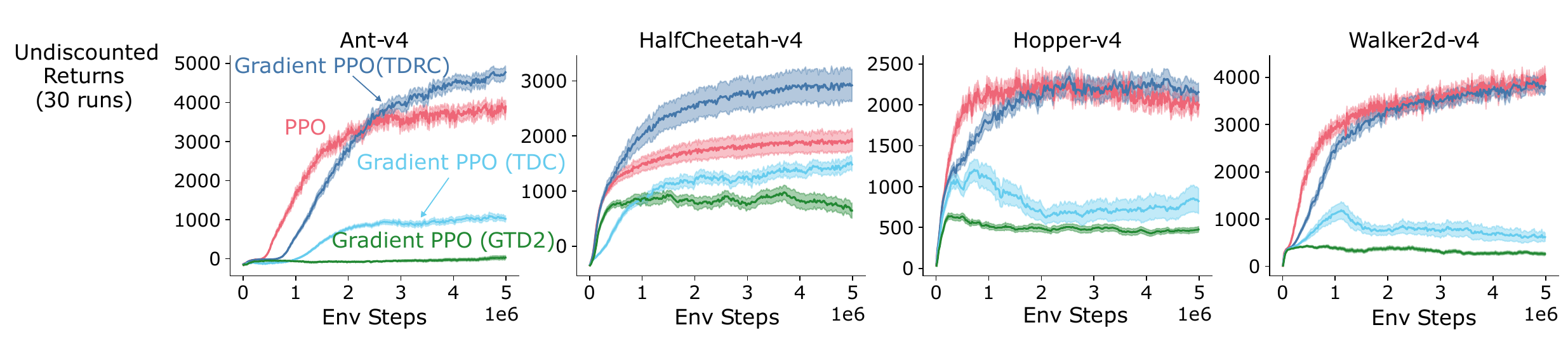}
    \caption{Gradient PPO variations evaluated on $4$ MuJoCo environments. The solid lines are the mean performance averaged over $30$ seeds, and the shaded area is the standard error.}
    \label{fig:ppo_mujoco_var}
\end{figure}

\begin{table}[b]
    \caption{Forward- and backward-view updates of our three proposed Gradient TD($\lambda$) algorithms for prediction with nonlinear function approximation.}
    \begin{center}
        \renewcommand{\arraystretch}{1.5}  
        \begin{tabular}{llll}
            \toprule
            {\bf Algorithm} & {\bf View} & $\Delta \vw_t$ & $\Delta \vtheta_t$ \\
            \midrule
            \multirow{2}{*}{GTD2($\lambda$)}
                & Forward & \green{$- H_t \grad{\vw} \delta^\lambda_t$} & \green{$(\delta^\lambda_t - H_t) \grad{\vtheta} H_t$} \\
                & Backward & \green{$-z^h_t \grad{\vw} \delta_t$} & \green{$\delta_t \vz^\vtheta_t - H_t \grad{\vtheta} H_t$} \\
            \midrule
            \multirow{2}{*}{TDC($\lambda$)}
                & Forward & \blue{$\delta^\lambda_t \grad{\vw} V_t$} \green{$- H_t \grad{\vw}$} \blue{$(V_t$}\green{ $+ \delta^\lambda_t)$} & \green{$(\delta^\lambda_t - H_t) \grad{\vtheta} H_t$} \\
                & Backward & \blue{$\delta_t \vz^\vw_t - H_t \grad{\vw} V_t$}\green{$ - z^h_t \grad{\vw} \delta_t$} & \green{$\delta_t \vz^\vtheta_t - H_t \grad{\vtheta} h_t$} \\
            \midrule
            \multirow{2}{*}{TDRC($\lambda$)}
                & Forward &\blue{$\delta^\lambda_t \grad{\vw} V_t$} \green{$- H_t \grad{\vw}$} \blue{($V_t$}\green{ $+ \delta^\lambda_t)$} & \green{$(\delta^\lambda_t - H_t) \grad{\vtheta} H_t $}\orange{$- \beta \vtheta_t$} \\
                & Backward & \blue{$\delta_t \vz^\vw_t - H_t \grad{\vw} V_t$}\green{$ - z^h_t \grad{\vw} \delta_t$} & \green{$\delta_t \vz^\vtheta_t - H_t \grad{\vtheta} H_t$} \orange{$ - \beta \vtheta_t$} \\
            \bottomrule
        \end{tabular}
    \end{center}
    \label{tab:prediction_algorithms}
\end{table}

\section{The Backward View for Gradient TD($\lambda$) Methods}
\label{sec:backward_view}
The forward-view algorithms we have derived so far have updates that depend on future information, making them unrealizable without the delay introduced by experience replay.
Alternatively, we can use eligibility traces via backward-view algorithms that incrementally generate the correct parameter updates on each time step.
We now derive the backward view algorithms for optimizing $\PBE$($\lambda$).

\textbf{GTD2}($\lambda$): As we prove below, the following backward-view updates are equivalent to the forward-view updates given in Eq. \ref{eq:b-gtd2_forward_main} and Eq. \ref{eq:b-gtd2_forward_aux}:
\begin{align}
    \label{eq:gtd2_backward_main}
    \Delta \vw_t &\defeq - z^h_t \grad{\vw} \delta_t
    \,,\\
    \label{eq:gtd2_backward_aux}
    \Delta \vtheta_t &\defeq \delta_t \vz^\vtheta_t - H_t \grad{\vtheta} H_t
    \,,
    \intertext{where, for $z^h_{-1} \defeq 0$, and $\vz^\vtheta_{-1} \defeq \vzero$, }
    \label{eq:z_w_sca}
    z^h_t &\defeq \gl z^h_{t-1} + H_t
    \,,\\
    \label{eq:gtd2_trace_aux}
    \vz^\vtheta_t &\defeq \gl \vz^\vtheta_{t-1} + \grad{\vtheta} H_t
    \,,
\end{align}
We show in the following theorem that this backward-view algorithm generates the same total parameter updates as the forward view under standard assumptions.
\begin{restatable}[]{theorem}{gtdfbequiv}
    \label{theorem:gtd2_fb_equiv}
    Assume the parameters $\vw$ and $\vtheta$ do not change during an episode of environment interaction.
    The forward and backward views of \mbox{GTD2($\lambda$)} are equivalent in the sense that they generate equal total parameter updates:
    \begin{align}
        \label{eq:fb_equiv_main}
        \sum_{t=0}^\infty H_t \grad{\vw} \delta^\lambda_t
        &= \sum_{t=0}^\infty z^h_t \grad{\vw} \delta_t
        \,,\\
        \label{eq:fb_equiv_aux}
        \sum_{t=0}^\infty (\delta^\lambda_t - H_t) \grad{\vtheta} H_t
        &= \sum_{t=0}^\infty (\delta_t \vz^\vtheta_t - H_t \grad{\vtheta} H_t)
        \,.
    \end{align}
\end{restatable}
\begin{proof}
    See \Cref{app:gtd2_proof}.
\end{proof}
\textbf{TDC($\lambda$):}
Let us slightly rewrite $\Delta \vw_t$ from \Cref{eq:tdc_main} in the following way:
\begin{equation}\label{eq:tdc_forward}
    \underbrace{\delta^\lambda_t \grad{\vw} V_t}_{\text{TD($\lambda$)}}
    ~+~
    \underbrace{(- H_t \grad{\vw} V_t)}_{\substack{\text{instantaneous} \\ \text{correction}}}
    ~+~
    \underbrace{(- H_t \grad{\vw} \delta^\lambda_t)}_{\text{GTD2($\lambda$)}}
    .
\end{equation}
We see that $\Delta \vw_t$ from \Cref{eq:tdc_forward} decomposes into three terms:
forward-view semi-gradient TD($\lambda$) with off-policy corrections;
an instantaneous correction that does not require eligibility traces;
and GTD2($\lambda$)'s term for $\Delta \vw_t$, for which we already derived and proved a backward-view equivalence in \Cref{theorem:gtd2_fb_equiv}.
As a consequence, we immediately deduce that the backward view for TDC($\lambda$) is
\begin{align}
    \Delta \vw_t &\defeq \delta_t \vz^\vw_t - H_t \grad{\vw} V_t - z^h_t \grad{\vw} \delta_t
    ,
    \intertext{where}
    \vz^\vw_t &\defeq \gl \vz^\vw_{t-1} + \grad{\vw} V_t
    ,
\end{align}
and $z^h_t$ is the same as before in \Cref{eq:z_w_sca}.
$\Delta \vtheta_t$ is generated by \Cref{eq:gtd2_backward_aux}.

\textbf{TDRC($\lambda$):}
Likewise, the regularized backward-view $\vtheta$ update is
\begin{equation}
    \Delta \vtheta_t \defeq \delta_t \vz^\vtheta_t - H_t \grad{\vtheta} H_t - \beta \vtheta_t
    ,
\end{equation}
where $\vz^\vtheta_t$ is once again generated by \Cref{eq:gtd2_trace_aux}.
\Cref{tab:prediction_algorithms} summarizes the forward view and the backward view for all the algorithms introduced. We highlighted the update components that arise from directly taking the gradient of $\PBE(\lambda)$ in \green{green}, the gradient correction components in \blue{blue}, and the regularization component in \orange{orange}.

Finally, we note that the backward view algorithms presented here do indeed update on every step, unlike PPO, but the proof above only shows equivalence at the end of the episode, like the original forward-backward equivalence of TD($\lambda$).

\section{QRC($\lambda$): Using the Backward View in Deep RL}
\label{sec:qrc}
In this section,  we extend the backward-view methods to action values and present three control algorithms based on three backward-view updates presented earlier. Since these algorithms are based on the backward view, they can make immediate updates without delay. Hence, they can work effectively in settings where it is prohibitive to have a large experience replay buffer (i.e., on-edge devices and mobile robots). Additionally, unlike forward-view methods, which require us to use a truncated version of the updates, backward-view methods do not have this limitation.

\subsection{QRC($\lambda$)}
Extending the backward-view algorithms to action values is straightforward. Here, we present the extensions to Q($\lambda$), but similar extensions can be done to other action-value methods, such as SARSA($\lambda$). Note that similar changes can be made to action-value methods using the forward view. 

Consider an action-value network parameterized by $\vw$, and write the TD error as:
\begin{equation}
  \delta_t = R_{t+1} + \gamma \max_{a' \in \mathcal{A}}\hat{q}(S_{t+1},a',\vw_t) -\hat{q}(S_t,A_t,\vw_t) \,.   
  \label{eq:td_q}
\end{equation}
The gradient of the TD error becomes the following:
\begin{equation}
  \nabla_{\vw_t}\delta_t = \gamma \nabla_{\vw_t} \Big(\max_{a' \in 
\mathcal{A}}\hat{q}(S_{t+1},a',\vw_t)\Big) - \nabla_{\vw_t}\hat{q}
(S_t,A_t,\vw_t) \,.
\label{eq:td_q_grad}
\end{equation}
The auxiliary function for $h$ is now predicting a function of both the states and actions: $H_t \defeq h(s_t, a_t, \vtheta_t)$. Using these modifications, we can now write the updates for the control variant of TDRC($\lambda$), which we refer to as QRC($\lambda$):
\begin{align}
\begin{split}
    \vz_{t}^{\vw} &= \gamma \lambda \vz_{t-1}^{\vw} + \nabla_{\vw_t}\hat{q}(S_t,A_t,\vw_t)\\
    z_{t}^{h} &= \gamma \lambda z_{t-1}^{h} + H_t\\
    \vz_{t}^{\vtheta} &= \gamma \lambda \vz_{t-1}^{\vtheta} + \nabla_{\vtheta}H_t\\
    \Delta \vw_t &= \delta_t \vz^\vw_t - H_t \grad{\vw} \hat{q}(S_t,A_t,\vw_t) - z^h_t \grad{\vw} \delta_t\\
    \Delta \vtheta_t &= \delta_t \vz^\vtheta_t - H_t \grad{\vtheta} H_t - \beta \vtheta_t\\
\end{split}
\label{eq:traces_q}
\end{align}
We can modify these updates to get QC($\lambda$), an update based on TDC($\lambda$), by simply setting $\beta = 0$. We can also get GQ($\lambda$), an update based on GTD2($\lambda$) by setting $\beta = 0$ and removing the gradient correction term (see \Cref{tab:prediction_algorithms}). Finally, we follow Watkins' Q($\lambda$) in that we decay the traces as described in the previous equations when a greedy action is selected and reset the traces to zero when a non-greedy action is selected~\citep{watkins1989learning}.

\subsection{Empirical Investigation of QRC($\lambda$)}
We evaluated the performance of QRC($\lambda$) across all the environments from the MinAtar benchmark~\citep{young2019minatar}. We compared the performance with Watkin's Q($\lambda$)~\citep{watkins1989learning} and StreamQ algorithm~\citep{elsayed2024streaming}, a recent algorithm combining Q($\lambda$) with a new optimizer and an initialization scheme for better performance in streaming settings.

For Q($\lambda$) and QRC($\lambda$), we used SGD and performed a hyperparameter sweep for different values for the step size and $\lambda$. We tested each hyperparameter configuration in all environments and across $5$ seeds. We then selected a single hyperparameter configuration that worked well across all environments, and we evaluated it for 30 more seeds in all environments. We provide the ranges and the final hyperparameters we used in Appendix~\ref{app:minatar}. For StreamQ, we used the hyperparameters suggested by the paper and the accompanying code, and we repeated the experiments for $30$ seeds in all environments. Figure~\ref{fig:minatar} shows the performance of all three algorithms across the $5$ MinAtar environments, and in all environments, QRC($\lambda$) outperforms both StreamQ and Q($\lambda$).

\begin{figure}[htb]
    \centering
    \includegraphics[width=\linewidth]{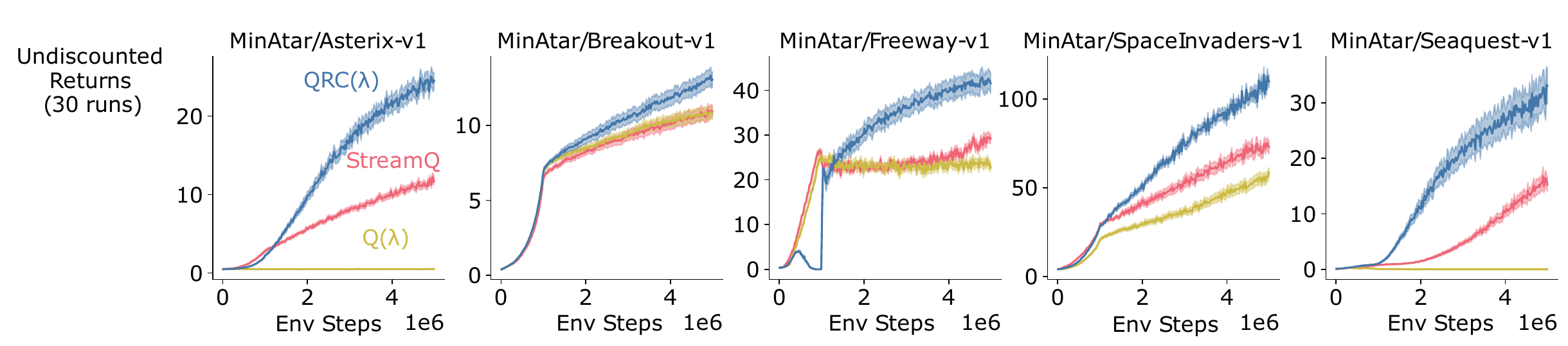}
    \caption{QRC($\lambda$), Q($\lambda$) and StreamQ algorithms evaluated on the five MinAtar environments. The solid lines are the mean performance averaged over $30$ seeds, and the shaded regions are the corresponding standard errors.}
    \label{fig:minatar}
\end{figure}

We evaluated the other two gradient-based algorithms, QC($\lambda$) and GQ2($\lambda$).
\Cref{fig:minatar_all} shows the results of this evaluation.
The results are consistent with forward-view results in \Cref{sec:gppo} in that having both the gradient correction and the regularization is needed for better performance.
However, here the regularization is not as critical as it was for Gradient PPO. 

\begin{figure}[htb]
    \centering
    \includegraphics[width=\linewidth]{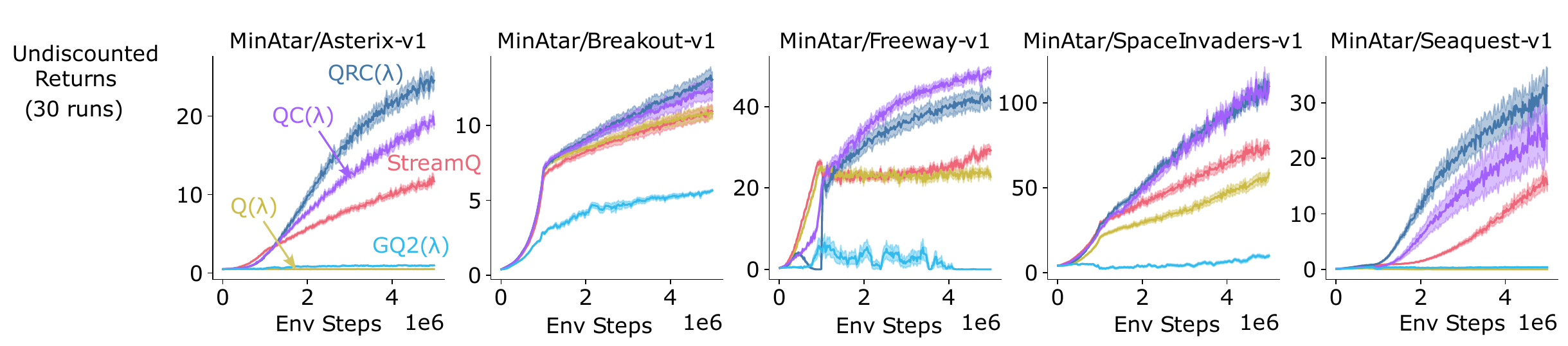}
    \caption{All gradient-based backward view algorithms evaluated on the $5$ MinAtar environments. The solid lines are the mean performance averaged over $30$ seeds, and the shaded regions are the corresponding standard errors.}
    \label{fig:minatar_all}
\end{figure}

\section{Conclusion}
We proposed the $\PBE$($\lambda$) objective, a multistep generalization of the Generalized Projected Bellman Error \citep{patterson2022generalized} based on the $\lambda$-return. We derived three algorithms for optimizing the new objective both in the forward view and in the backward view. Of the three algorithms we developed, we showed that TDRC($\lambda$) is stable, fast, and results in a high-quality solution. We introduced two Deep RL algorithms that use the newly derived update rules, and we showed that our new algorithms outperform both PPO with a buffer and streaming algorithms without replay buffers. Further work remains to verify the convergence guarantees for TDC($\lambda$) and TDRC($\lambda$), and extend the gradient-based updates to more Deep RL algorithms.


\appendix



\section{Proof of \Cref{theorem:gtd2_fb_equiv}}
\label{app:gtd2_proof}

\gtdfbequiv*
In the proof below, we added importance sampling for generality.
\begin{proof}
    We start by showing \Cref{eq:fb_equiv_main} holds.
    Note that
    \begin{align}
        H_t \grad{\vw} \delta^\lambda_t
        &= H_t \rho_t \grad{\vw} \delta_t + H_t \gl \rho_t \rho_{t+1} \grad{\vw} \delta_{t+1} + H_t (\gl)^2 \rho_t \rho_{t+1} \rho_{t+2} \grad{\vw} \delta_{t+2} \dots
        .
    \end{align}
    The total sum of these forward-view contributions is therefore
    \begin{align}
        \nonumber
        \sum_{t=0}^\infty H_t \grad{\vw} \delta^\lambda_t
        &= (H_0 \rho_0 \grad{\vw} \delta_0 + H_0 \gl \rho_0 \rho_1 \grad{\vw} \delta_1 + \dots)
        + (H_1 \rho_1 \grad{\vw} \delta_1 + H_1 \gl \rho_1 \rho_2 \grad{\vw} \delta_2 + \dots)
        + \dots \\
        &= (H_0 \rho_0) \grad{\vw} \delta_0 + (H_0 \gl \rho_0 \rho_1 + H_1 \rho_1) \grad{\vw} \delta_1
        + \dots \\
        &= z^h_0 \grad{\vw} \delta_0 + z^h_1 \grad{\vw} \delta_1 + \dots \\
        &= \sum_{t=0}^\infty z^h_t \grad{\vw} \delta_t
        ,
    \end{align}
    which proves \Cref{eq:fb_equiv_main}.
    Next, consider \Cref{eq:fb_equiv_aux}.
    Notice that the equality holds if and only if
    \begin{equation}
        \label{eq:fb_aux_helper}
        \sum_{t=0}^\infty {\delta}^\lambda_t \grad{\vtheta} H_t
        = \sum_{t=0}^\infty \delta_t z^\vtheta_t
        ,
    \end{equation}
    and further note that
    \begin{align}
        \delta^\lambda_t \grad{\vtheta} H_t
        &= \rho_t \delta_t \grad{\vtheta} H_t + \gl \rho_t \rho_{t+1} \delta_{t+1} \grad{\vtheta} H_t  + (\gl)^2 \rho_t \rho_{t+1} \rho_{t+2} \delta_{t+2} \grad{\vtheta} H_t + \dots
        .
    \end{align}
    The total sum of these forward-view contributions is therefore
    \begin{align}
        \sum_{t=0}^\infty \delta^\lambda_t \grad{\vtheta} H_t
        &= (\rho_0 \delta_0 \grad{\vtheta} H_0 + \gl \rho_0 \rho_1 \delta_1 \grad{\vtheta} H_0 + \dots)
        + (\rho_1 \delta_1 \grad{\vtheta} H_1 + \gl \rho_1 \rho_2 \delta_2 \grad{\vtheta} H_1 + \dots)
        + \dots \\
        &= \delta_0 (\rho_0 \grad{\vtheta} H_0) + \delta_1 (\gl \rho_0 \rho_1 \grad{\vtheta} H_0 + \rho_1 \grad{\vtheta} H_1) + \dots \\
        &= \delta_0 z^\vtheta_0 + \delta_1 z^\vtheta_1 + \dots \\
        &= \sum_{t=0}^\infty \delta_t z^\vtheta_t
        ,
    \end{align}
    which establishes \Cref{eq:fb_aux_helper} to prove \Cref{eq:fb_equiv_aux} and complete the proof.
\end{proof}




\subsubsection*{Acknowledgments}
\label{sec:ack}
We would like to thank NSERC, CIFAR, and Amii for research funding and the Digital Research Alliance of Canada for the
computational resources. We would also like to thank Vlad Tkachuk for providing feedback on the paper draft.



\bibliography{main}
\bibliographystyle{rlj}

\beginSupplementaryMaterials

%
%
\section{Gradient TD($\lambda$) with Importance Sampling Correction}\label{app:is_correc}
We now discuss the modifications needed when the experiences $(S_t,A_t,R_t,S_{t+1})$ are collected by a behaviour policy $b$ rather than the target policy $\pi$.
Letting $\rho_t = \frac{\pi(A_t|S_t)}{b(A_t|S_t)}$ be the importance sampling ratio at time $t$, we can scale the TD error by this factor to form a bias-corrected TD error $\hat{\delta}_t \defeq \rho_t \delta_t$, since $\E_b[\rho_t \delta_t | S_t=s] = \E_\pi[\delta_t | S_t=s] = \vdelta(s)$ \citep{precup2000eligibility}.
By induction, it follows that the bias-corrected TD($\lambda$) error is
\begin{equation}
    \label{eq:td_lambda_hat}
    \hat{\delta}^\lambda_t
    \defeq \sum_{i=0}^\infty (\gl)^i \! \left(\prod_{j=0}^i \rho_{t+j}\right) \! \delta_{t+i}
    = \rho_t (\gl \hat{\delta}^\lambda_{t+1} + \delta_t).
\end{equation}

The backward-view traces will then be defined as follows:
\begin{align}
    z^h_t &\defeq \rho_t(\gl z^h_{t-1} + h_t)
    \,,\\
    \vz^\vtheta_t &\defeq \rho_t(\gl \vz^\vtheta_{t-1} + \grad{\vtheta} h_t)
    \,,
\end{align}
\section{Forward-View Algorithms}\label{app:algos}
In this section, we provide the pseudocode for the forward-view algorithms.
We start with PPO in algorithm~\ref{alg:full_ppo}. PPO alternates between two main components: collecting a fixed-length trajectory of interactions using the current policy, and performing several steps of gradient updates using the collected trajectory.
The gradient updates involve updating the value function towards an estimate of the $\lambda$-return based on the collected trajectory, and updating the policy parameters using the log-likelihood ratio. These steps are illustrated in \Cref{alg:full_ppo}. In the algorithm, we refer to the policy and value function parameters used during the collection of the trajectory as $\otheta$ and $\oldw$, respectively, while we refer to the most recent policy and value parameters as $\neww$ and $\ntheta$, those would be the result of the most recent mini-batch update.

\Cref{alg:gppo} shows the modifications needed to combine PPO with TDRC($\lambda$) to produce Gradient PPO. We highlight the main changes over the PPO algorithm in \blue{blue}. Gradient PPO introduces three new parameters: 1) Truncation length, $T$, which represents the sequence length used to compute the $\lambda$-returns. 2) Learning rate for the auxiliary variable $h$, $\alpha_h$. 3) Regularization coefficient $\beta$, we found that simply setting $\beta = 1$ worked well for all the experiments we presented.
Additionally, for the gradient updates, we construct a mini-batch of sequences and estimate the $\lambda$-returns per sequence. Note that a major change over PPO is that the $\lambda$-returns are estimated per minibatch using the latest parameters rather than stale estimates. As mentioned in the main paper, this change allows us to take the gradient of the $\delta^{\lambda} $ with respect to the latest parameters, which is needed for updates of the GTD algorithms.

We can write two loss functions that correspond to the parameter updates in \Cref{alg:gppo}. These losses can be implemented in most machine learning libraries, such as PyTorch \citep{paszke2019pytorch} and Jax \citep{jax2018github}, where it is sometimes easier to implement a loss function instead of the parameter updates.
We can write an objective based on TDRC($\lambda$) as follows:
\begin{equation}
    \label{eq:gppo_1}
    L_t(\vw_t) = \hat{h}(S_t,\vtheta_t)\delta^\lambda_{t:T} - \text{sg}\left(\delta^\lambda_{t:T} - \hat{h}(S_t,\vtheta_t)\right) \hat{v}(S_t,\vw_t).
\end{equation}
Where $\text{sg}$ refers to a stop gradient operation. Minimizing $L_t(\vw_t)$ results in a parameter update equivalent to \Cref{eq:tdc_main}.

We can also write an objective function for the auxiliary variable $\hat{h}$, which can be written as:
\begin{equation}
    \label{eq:gppo_2}
     L_t(\vtheta_t) = \frac{1}{2}\left({\delta}^\lambda_{t:T} - \hat{h}(S_t,\vtheta_t)\right)^2 + \frac{\beta}{2} \|\vtheta_t\|^2.
\end{equation}
Minimizing $L_t(\vtheta_t)$ results in a parameter update equivalent to \Cref{eq:tdrc_theta}.

\begin{algorithm}[!hbt]
    \caption{PPO Algorithm}~\label{alg:full_ppo}
    \begin{algorithmic}
        \State{Input: a differentiable policy parametrization $\pi(a|s,\pmb{\theta})$}
        \State{Input: a differentiable state-value function parametrization $\hat{v}(s,\mathbf{w})$}
        \State{Algorithm parameters: learning rate $\alpha$, rollout length $\tau$, mini-batch size $n$, number of epochs $k$, value coefficient $c_1$, entropy coefficient $c_2$, clip coefficient $\epsilon$, max gradient norm $c$.}
        \For{iteration $= 1, 2, \cdots$}
            \State{Run $\pi_{\text{old}}(a|s,\otheta)$ for $\tau$ steps.}\Comment{ Collect a trajectory of interactions}
            \State{Calculate $\hat{v}(s_{t+\tau},\oldw)$} \Comment{For bootstrapping}
            \State{Set $\textit{\^A}_{t+\tau}^{(\gamma,\lambda)} = 0$} \Comment{initialization GAE estimate.}
            \For{$j=t+\tau -1$, \ldots, t} \Comment{Calculating GAE using the collected trajectory of interactions.}
            \State{$\delta_j = R_{j+1} + \gamma \hat{v}(s_{j+1},\oldw) - \hat{v}(s_{j},\oldw)$}
            \State{$\textit{\^A}_j^{(\gamma,\lambda)} = \delta_j + \gamma \lambda \textit{\^A}_{j+1}^{(\gamma,\lambda)}$}
            \State{$\hat{G}_{j}^\lambda = \textit{\^A}_j^{(\gamma,\lambda)} + \hat{v}(s_{j},\oldw)$}
            \EndFor{}
            \For{epoch = 1, \ldots, k} \Comment{Learning}
                \State{Shuffle the transitions}
                \State{Divide the data into $m$ mini-batches of size $n$, where $m = \tau /n$.}
                \For{mini-batch $= 1, \ldots, m$ }
                    \State{Calculate: $\log \pi_{new}(a|s, \ntheta)$, $\hat{v}(s,\neww)$ for samples in the mini-batch.}
                    \State{Normalize $\textit{\^A}^{(\gamma,\lambda)}$ estimates per batch.}
                    \State{\textbf{Policy objective:} $L_{\text{p}} = - \frac{1}{n}\sum_{j=1}^{n} \min(r_{j}\textit{\^A}^{(\gamma,\lambda)}_{j,\oldw} ,${$ \text{clip}_{\epsilon}(r_{j}) \textit{\^A}_{j,\oldw}^{(\gamma,\lambda)} )$},}
                    \State{where $r_j= \frac{\pi(a_j|s_j,\ntheta)}{\pi(a_j|s_j,\otheta)}$ {and $\text{clip}_{\epsilon}(r_j) = \text{clip}(r_j, 1-\epsilon,1+\epsilon)$.}}

                    \State{\textbf{Value objective:} {$L_{\text{v}} = \frac{1}{n}\sum_{j=1}^{n} \max({(\hat{v}(s_{j},\neww) - \hat{G}^{\lambda}_{j,\oldw})}^{2},${$ {(\text{clip}_{\epsilon}(\hat{v}) - \hat{G}^{\lambda}_{j})}^{2})$}},
                    \State{where $\text{clip}_{\epsilon}(\hat{v}) = \text{clip}(\hat{v}(s_{j},\neww), 1-\epsilon,1+\epsilon)$}}
                    
                    \State{\textbf{Calculate the entropy of the policy:} $L_s =\frac{1}{n}\sum_{j=1}^{n}S(\pi(s_j,\ntheta))$}
                    \State{\textbf{Calculate the total loss:} $L = L_{\text{p}} + c_{1}L_{\text{v}} - c_{2} L_s$}
                    \State{\textbf{Calculate the gradient }$\hat{g}$}
                    \If{$\lVert\hat{g}\rVert > c$ }
                    \State{$\hat{g} \gets \frac{c}{\lVert\hat{g}\rVert}\hat{g}$}
                    \EndIf{}
                    \State{Update the parameters using the gradient to minimize the loss function.}
                \EndFor{}    
            \EndFor{}
        \EndFor{}
    \end{algorithmic}
  \end{algorithm}

\begin{algorithm}[hbt!]
    \caption{Gradient PPO: PPO with TDRC($\lambda$) }~\label{alg:gppo}
    \begin{algorithmic}
        \State{Input: a differentiable policy parametrization $\pi(a|s,\vtheta)$}
        \State{Input: a differentiable state-value function parametrization $\hat{v}(s,\mathbf{w})$}
        \State{Input: a differentiable auxiliary function parametrization $\hat{h}(s,\vtheta_h)$}
        \State{Algorithm parameters: learning rate $\alpha$, rollout length $\tau$, mini-batche size $n$, number of epochs $k$, entropy coefficient $c_2$, clip coefficient $\epsilon$, max gradient norm $c$, \blue{Truncation Length $T$, $h$ learning rate $\alpha_h$, regularization coefficient $\beta = 1$ }}
        \For{iteration $= 1, 2, \cdots$}
            \State{Run $\pi_{\text{old}}(a|s,\otheta)$ for $\tau$ steps.}\Comment{ Collect a trajectory of interactions}
            \State{Calculate $\hat{v}(s_{t+\tau},\oldw)$} \Comment{For bootstrapping}
            \blue{
            \State{Construct a batch of $\frac{\tau}{T}$ sequences, where each sequence is: \\
            $\langle s_i, a_{i+1} , R_{i+1}, \log \pi_{old}(a_{i+1}|s_i,\otheta), \hat{v}(s_{i},\oldw) \rangle, \ldots$\\$ \langle s_{i+T}, a_{i+T+1} , R_{i+T+1}, \log \pi_{old}(a_{i+T+1}|s_{i+T},\otheta), \hat{v}(s_{i+T},\oldw)\rangle$    } }
            \For{epoch = 1, \ldots, k} \Comment{{Learning}}
                \State{Shuffle the sequences}
                \State{Divide the data into $m$ mini-batches of size $n$, where $m = \tau /(n*T)$.}
                \For{mini-batch $= 1, \ldots, m$ }
                    \blue{
                        \For{$j=t+T -1$, \ldots, t} 
                        \Comment{This loop is parallelized over the sequences.}
                            \State{$\delta_j = R_{j+1} + \gamma \hat{v}(s_{j+1},\neww) - \hat{v}(s_{j},\neww)$}
                            \State{$\nabla{\delta_j}_{\neww} = R_{j+1} + \gamma \nabla{\hat{v}}(s_{j+1},\neww) - \nabla{\hat{v}}(s_{j},\neww)$}
                            \State{$\delta_j^{\lambda} = \delta_j + \gamma \lambda 
                            \delta_{j+1}^{\lambda}$}
                            \State{$\nabla{\delta}_j^{\lambda} = \nabla{\delta}_j + \gamma \lambda 
                            \nabla{\delta}_{j+1}^{\lambda}$}
                        \EndFor{}
                        }

                    \State{Calculate: $\log \pi_{new}(a|s, \ntheta)$ for samples in the mini-batch.}
                    \State{\textbf{Policy objective:} $L_{\text{p}} = - \frac{1}{n}\sum_{j=1}^{n} \min(r_{j}\delta^{\lambda}_{j} ,${$ \text{clip}_{\epsilon}(r_{j}) \delta^{\lambda}_{j} )$}}
                    \State{where $r_j= \frac{\pi(a_j|s_j,\ntheta)}{\pi(a_j|s_j,\otheta)}$, {and $\text{clip}_{\epsilon}(r_j) = \text{clip}(r_j, 1-\epsilon,1+\epsilon)$}}
                    \State{\textbf{Calculate the entropy of the policy:} $L_S = \frac{1}{n}\sum_{j=1}^{n}S(\pi(s_j,\ntheta))$}
                    \State{\textbf{Calculate the total loss:} $L = L_{\text{p}}  - c_{2} L_S$}
                    \State{\textbf{Calculate the gradient }$\hat{g}$}
                    \If{$\lVert\hat{g}\rVert > c$ }
                    \State{$\hat{g} \gets \frac{c}{\lVert\hat{g}\rVert}\hat{g}$}
                    \EndIf{}
                    \State{Update the policy using the gradient to minimize the loss function.}
                    \blue{
                    \State{Update value parameters using the following update:}
                    \State{$\delta^\lambda_t \grad{\vw} v_t - h_t \grad{\vw} (v_t + \delta^\lambda_t)$}
                    $(\delta^\lambda_t - h_t) \grad{\vtheta} h_t - \beta \vtheta_{h,t}$ \\
                    \State{Update $h$ parameters using the following update:}
                    \State{$(\delta^\lambda_t - h_t) \grad{\vtheta} h_t - \beta \vtheta_{h,t}$ }
                    }
                    
                \EndFor{}    
            \EndFor{}
        \EndFor{}
    \end{algorithmic}
  \end{algorithm}

\subsection{Experimental Details of PPO and Gradient PPO}\label{app:gppo}
For PPO, we used the default hyperparameters widely used for PPO, which reproduce the best-reported performance for PPO on MuJoCo \citep {shengyi2022the37implementation}. We include those hyperparameters in \Cref{tab:ppo_hypers} for completeness.

For Gradient PPO, we first performed a hyperparameter sweep for $\lambda$, actor learning, and critic learning rate. We show the ranges used for the sweep in Table~\ref{tab:gppo_hypers_ranges}. We repeated the experiments for each hyperparameter configuration in that sweep over $5$ seeds. Based on that sweep, we chose the hyperparameters that generally performed well across all environments, Table~\ref{tab:gppo_hypers}. Finally, we fixed those hyperparameters configurations for all environments and ran the algorithm again for $30$ seeds using those hyperparameters. 

\begin{figure}[htb]
    \centering
    \includegraphics[width=\linewidth]{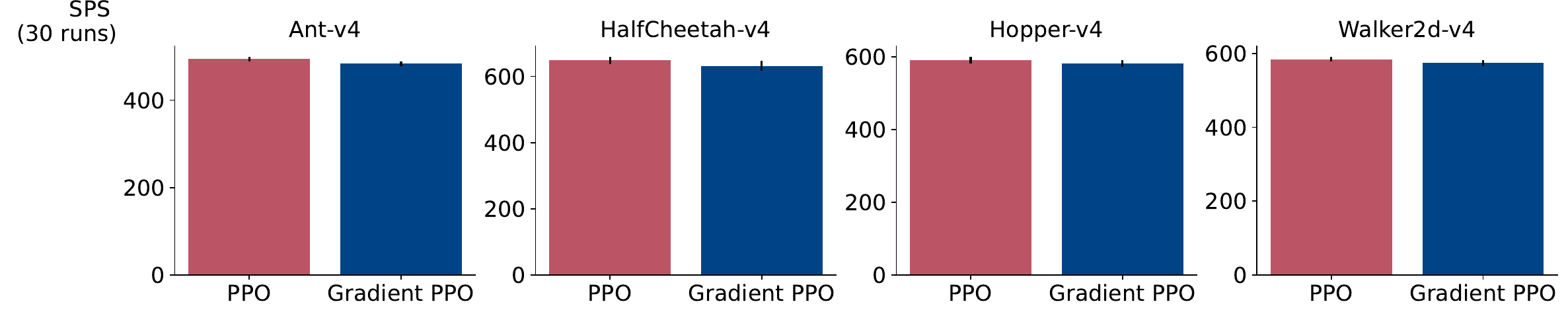}
    \caption{SPS for Gradient PPO and PPO evaluated on four MuJoCo environments. The bars indicate the mean across $30$ runs and the black bars indicate the standard error.}
    \label{fig:ppo_sps}
\end{figure}

Finally, to show that the additional calculations don't affect the run time of Gradient PPO, we plotted the Steps Per Second (SPS) for both PPO and Gradient PPO across all environments. \Cref{fig:ppo_sps} shows that the SPS values are almost the same moving from PPO to Gradient PPO.

\begin{table}[H]
    \centering
    \begin{tabular}{l l}
        \hline
        Name & Default Value \\
        \hline
        Policy Network & (64, tanh, 64, tanh, Linear) + Standard deviation variable \\
        Value Network & (64, tanh, 64, tanh, Linear) \\
        Buffer size & 2048 \\
        Num epochs & 4 \\
        Mini-batch size & 64 \\
        GAE, $\lambda$ & 0.95 \\
        Discount factor, $\gamma$ & 0.99 \\
        Clip parameter & 0.2 \\
        Input Normalization & True \\
        Advantage Normalization & True \\
        Value function loss clipping & True \\
        Max Gradient Norm & 0.5 \\
        Optimizer & Adam \\
        Actor step size & 0.0003 \\
        Critic step size & 0.0003 \\
        Optimizer $\epsilon$ & $1 \times 10^{-5}$ \\
        \hline
    \end{tabular}
    \caption{The default hyperparameters used for PPO. The hyperparameter values are based on the implementation details by \cite{shengyi2022the37implementation}.}
    ~\label{tab:ppo_hypers}
\end{table}

\begin{table}[H]
    \centering
    \begin{tabular}{l l}
        \hline
        Name & Default Value \\
        \hline
        Policy Network & (64, tanh, 64, tanh, Linear) + Standard deviation variable \\
        Value Network & (64, tanh, 64, tanh, Linear) \\
        Buffer size & 2048 \\
        Num epochs & 4 \\
        Mini-batch size & 256 (split into $8$ sequences of length $32$)\\
        $\lambda$ & 0.95 \\
        Discount factor, $\gamma$ & 0.99 \\
        Clip parameter & 0.2 \\
        Input Normalization & True \\
        Advantage Normalization & True \\
        Max Gradient Norm & 0.5 \\
        Optimizer & Adam \\
        Actor step size & 0.0003 \\
        Critic step size & 0.003 \\
        h step size & 0.003 \\
        regularization coef, $\beta$ & 1.0 \\
        Optimizer $\epsilon$ & $1 \times 10^{-5}$ \\
        \hline
    \end{tabular}
    \caption{The hyperparameters used for Gradient PPO in all the MuJoCo experiments.}
    \label{tab:gppo_hypers}
\end{table}

\begin{table}[b]
    \centering
    \begin{tabular}{l l}
        \hline
        Name & Sweep Range \\
        \hline
        $\lambda$ & $[0.7,0.8,0.9,0.95]$ \\
        Actor step size & $[0.001,0.003,0.0001,0.0003,0.00001,0.00003]$ \\
        Critic step size & $[0.001,0.003,0.0001,0.0003,0.00001,0.00003]$ \\
        regularization coef, $\beta$ & $[1.0,0.0]$ \\
        \hline
    \end{tabular}
    \caption{Hyperparameter ranges used for the sweep experiments for Gradient PPO.}
    \label{tab:gppo_hypers_ranges}
\end{table}
\section{Backward-View Algorithms}
In this section, we provide the pseudocode and the hyperparameters details for the backward-view algorithms.
\Cref{alg:qrc_lambda} shows the pseudocode for QRC($\lambda$), where at each timestep, the agent samples an action from an $\epsilon$-greedy policy and takes one step in the environment to observe the subsequent reward and next state. Then, based on that transition, it makes an update to its parameters using the traces it is carrying and also updates those traces.
\begin{algorithm}[!hbt]
    \caption{QRC($\lambda$) Algorithm}~\label{alg:qrc_lambda}
    \begin{algorithmic}
        \State{Input: a differentiable state-value function parametrization $\hat{q}_\mathbf{w}$}
        \State{Input: a differentiable auxiliary function parametrization $\hat{h}_\mathbf{\theta}$}
        \State{Algorithm parameters: learning rate $\alpha_q$, $h$ learning rate $\alpha_h$,  exploration parameter $\epsilon$.}
        \State{Initialize $\vz^\vtheta_t \leftarrow \mathbf{0}$}
        \State{Initialize $\vz_{t}^{\vw} \leftarrow \mathbf{0}$}
        \State{Initialize $z^{h}_t \leftarrow 0$}
        \State{Observe initial state $S_0$}
        \For{iteration $t = 1, 2, \cdots$}
            \State{Sample an action $A_t \sim \pi$ }.\Comment{We use an $\epsilon$-greedy policy.}
            \State{Take action $A_t$, observe $R_{t+1}$ and $S_{t+1}$.} 
            \State{Compute $\delta_t$ and $\nabla_{\vw_t}\delta_t$ according to Eq.~\ref{eq:td_q} and Eq.~\ref{eq:td_q_grad}, respectively.}
            \State{Update the traces $\vz^\vtheta_t$, $\vz_{t}^{\vw}$, and $z^{h}_t$ according to Eq.\ref{eq:traces_q}.}
            \State{Compute $\Delta \vw_t$ and $\Delta \vtheta_t$ according to Eq.\ref{eq:traces_q}.} \Comment{To use the other algorithmic variants, replace those equations with other backward-view algorithms in Table~\ref{tab:prediction_algorithms}.}
            \State{Update the parameters $\vw_{t+1} \leftarrow \vw_t + \alpha_q \Delta \vw_t$}
            \State{Update the parameters $\vtheta_{t+1} \leftarrow \vtheta_t + \alpha_h \Delta \vtheta_t$}
            \If{episode terminated or $A_t$ is non-greedy}
            \State{reset the traces $\vz^\vtheta_t$, $\vz_{t}^{\vw}$, and $z^{h}_t$ to zeros.}
            \EndIf
        \EndFor{}
    \end{algorithmic}
  \end{algorithm}

\subsection{Experimental Details of MinAtar}\label{app:minatar}
In our experiments with QRC($\lambda$), we used the same normalization wrappers and the sparse initialization proposed by~\cite{elsayed2024streaming}. However, we used SGD as the optimizer, as the optimizer presented~\cite{elsayed2024streaming} can't be easily mapped to our updates. These choices were made so that the difference in performance between those algorithms can be associated with the Gradient TD updates.

For QRC($\lambda$) and Q($\lambda$), we first performed a hyperparameter sweep over $\lambda$, and the learning rate for both the value network and the $h$ network, for QRC($\lambda$). We include the values used for the sweep in ~\ref{tab:minatr_q_sweep} and the final hyperparameters used in \Cref{tab:minatr_qrc_hypers}.

Finally, we estimated the SPS for both QRC($\lambda$) and Q($\lambda$), and the estimated values are shown in Figure~\ref{fig:qs_sps}. Since the backward-view algorithms do not have batch updates, we were not able to parallelize the additional computations as we did in the forward-view (i.e, Gradient PPO), and we notice a marginal runtime increase in run time for QRC($\lambda$) compared to Q($\lambda$) due to the additional computation required for the auxiliary variable $h$. 

\begin{figure}[H]
    \centering
    \includegraphics[width=\linewidth]{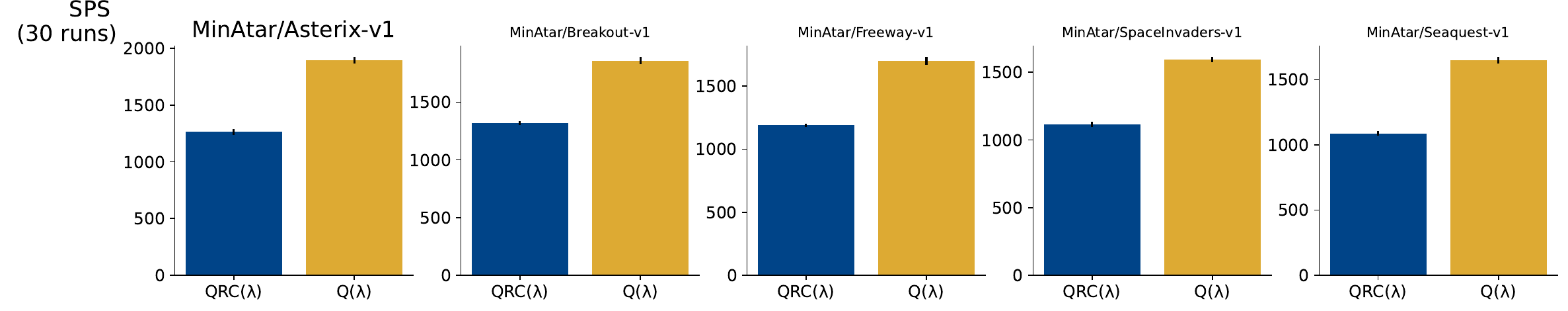}
    \caption{SPS for QRC($\lambda$) and Q($\lambda$) algorithms on MinAtar environments.}
    \label{fig:qs_sps}
\end{figure}

\begin{table}[H]
    \centering
    \begin{tabular}{l l}
        \hline
        Name & Default Value \\
        \hline
        $\lambda$ & [0.7,0.8,0.9,0.95] \\
        Optimizer & SGD \\
        step size & [0.001,0.0001,0.00001,0.000001] \\
        h step scale & [1.0,0.1] \\
        regularization coef, $\beta$ & [1.0,0.0] \\
        \hline
    \end{tabular}
    \caption{Hyperparameters ranges used for the sweep experiments in MinAtar.}
    ~\label{tab:minatr_q_sweep}
\end{table}

\begin{table}[H]
    \centering
    \begin{tabular}{l l}
        \hline
        Name & Default Value \\
        \hline
        $\lambda$ & 0.8 \\
        Input Normalization & True \\
        Optimizer & SGD \\
        step size & 0.0001 \\
        h step size & 1.0 \\
        regularization coef, $\beta$ & 1.0, for QRC($\lambda$), and 0.0, for QC($\lambda$) and GQ2($\lambda$). \\
        start exploration $\epsilon$, & 1.0 \\
        end exploration $\epsilon$, & 0.01 \\
        exploration fraction & 0.2 \\
        \hline
    \end{tabular}
    \caption{Final hyperparameters used for QRC($\lambda$) experiments with MinAtar}
    ~\label{tab:minatr_qrc_hypers}
\end{table}

\end{document}